\documentclass[11pt]{article}
\usepackage{float}
\usepackage[final]{acl}

\usepackage{times}
\usepackage{latexsym}

\usepackage[T1]{fontenc}

\usepackage[utf8]{inputenc}

\usepackage{microtype}

\usepackage{inconsolata}

\usepackage{graphicx}

\usepackage{amsmath}
\usepackage{amsfonts}
\usepackage{amsthm}
 \usepackage{booktabs}
 \usepackage{multirow}
 \usepackage{amsmath, amsthm} 
\usepackage{amsmath,amssymb,amsthm}    
\usepackage{booktabs}                   
\usepackage{tikz}                       
\usepackage{pgfplots}
\pgfplotsset{compat=1.18}  
\usepackage{tikz}
\usepackage{booktabs}      
\usepackage{hyperref}                   
\usepackage{url}                        

\usepackage{microtype}                  
\usepackage{caption}                    
\usepackage{enumitem}                   

\hypersetup{
  colorlinks=true,
  linkcolor=black,
  citecolor=blue,
  urlcolor=blue,
  pdfborder={0 0 0}
}
\newtheorem{proposition}{Proposition}
\newtheorem{assumption}{Assumption}
\title{SEAM: Spectral Embedding Alignment and Merging for Low-Resource TTS}

\author{Walaa Othman \\
  ITMO University \\
  SPC RAS \\
  \footnotesize{\url{walaa_othman@itmo.ru}} \\\And
  Baher Mohammad \\
  ITMO University \\
  MWS AI \\\And
  Stamatios Lefkimmiatis \\
  ITMO University \\
  MWS AI \\\And
  Ammar Ali \\
  ITMO University \\
  MWS AI \\
  \footnotesize{\url{ammarali32@itmo.ru}} \\}

\begin{document}
\maketitle

\begin{abstract}
Expanding the vocabulary of conditional text-to-speech (TTS) models typically requires extensive fine-tuning with paired audio, creating a major bottleneck for low-resource and cross-lingual adaptation. To address this, we introduce SEAM, a zero-audio initialization method that constructs embeddings for unseen tokens from known subtokens. Unlike standard pooling strategies that ignore sequential structure, SEAM applies a spectral filter to preserve temporal phonetic order and uses closed-form ridge regression to align the new embeddings with the pretrained conditioning manifold. Across three languages and two TTS architectures, SEAM approaches or matches baseline performance using up to four times less training data, reducing sample complexity by 75\% compared to random initialization and mean pooling. Our approach provides a lightweight, reproducible, and theoretically grounded solution for vocabulary expansion in low-resource speech synthesis.
\end{abstract}
\section{Introduction}
The rapid proliferation of Agentic AI has positioned text-to-speech (TTS) systems as foundational voice interfaces for autonomous, multimodal agents. As these agents scale to global deployments, the ability to rapidly adapt to new languages particularly low-resource and under-documented ones has become a critical bottleneck. Extending conditional TTS models to novel linguistic settings is fundamentally constrained by \textbf{vocabulary expansion}: every unseen phonetic or graphemic unit requires a corresponding embedding in the model's conditioning space. In practice, these tokens are initialized randomly and fine-tuned with extensive paired audio, a requirement that is infeasible when target-language data is scarce or costly to collect. This initialization bottleneck severely limits cross-lingual generalization and perpetuates a stark accessibility gap in speech technology.

We address this gap with \textbf{SEAM} (Spectral Embedding Alignment \& Merging), a principled, self-calibrating initialization method. SEAM treats subtoken embeddings as a discrete signal, applies spectral low-pass filtering to preserve low-frequency phonetic compositionality, and projects the result onto the pretrained embedding manifold via closed-form ridge regression. Grounded in the Lipschitz continuity of conditional generation, our approach is designed to promote acoustic stability and compositional faithfulness from the initialization stage.  Our contributions are:
1) A spectrally grounded merging operator that preserves temporal phonetic structure and suppresses boundary artifacts; 
2) A closed-form, self-calibration protocol that requires no gradient-based optimization or target-language audio;
3) Data-efficiency gains over standard initialization heuristics, enabling rapid vocabulary expansion for low-resource TTS.

\section{Related Work}
Recent cross-lingual and low-resource TTS research has focused on model-level adaptation strategies, including zero-shot multi-speaker transfer via external speaker encoders~\cite{yourtts2022}, cross-lingual phoneme sharing~\cite{phoneme_share_tts}, and meta-learning for rapid speaker alignment~\cite{meta_tts}. While these methods improve transferability, they typically assume access to substantial target-language data for embedding alignment or rely on gradient-based optimization that becomes unstable when data is extremely scarce. Consequently, the foundational step of vocabulary expansion remains underexplored. Current practice adapts heuristics from language modeling, most notably decomposing unseen tokens into known subtokens and averaging their embeddings~\cite{ruadapt2024}. Although computationally trivial, this permutation-invariant pooling discards the strict temporal order and articulatory dynamics essential to speech synthesis, forcing the model to relearn sequential structure from scratch during fine-tuning and inflating sample complexity. Our method provides an order-preserving, spectrally grounded initialization that aligns novel embeddings with the pretrained manifold, thereby stabilizing optimization and drastically reducing the target-language data required for subsequent adaptation.
To our knowledge, no prior work proposes an audio-guided initialization specifically for vocabulary expansion in autoregressive transformer-based TTS, making the zero-audio property of SEAM particularly important.
\section{Method}
\label{sec:method}

We address the problem of initializing conditioning embeddings for novel phonetic tokens in llm-diffusion based TTS systems. Let $\mathcal{V}_{\text{old}}$ denote a pretrained vocabulary with embedding lookup $E: \mathcal{V}_{\text{old}} \rightarrow \mathbb{R}^d$, and let $\Phi_\theta$ be the conditional generator. For a new token $t \in \mathcal{V}_{\text{new}}$ decomposable into known subtokens $\{s_1, \dots, s_m\} \subset \mathcal{V}_{\text{old}}$, we seek a conditioning vector $c_t \in \mathbb{R}^d$ that satisfies: (i) stability under the generative process, (ii) preservation of temporal phonetic order, and (iii) alignment with the pretrained embedding manifold. We achieve this via a spectral merging operator $\Psi$, calibrated in closed form.

\subsection{Continuity Principle}
\label{subsec:continuity}
Well-trained conditional TTS generators assume that the mapping from conditioning embeddings to audio is locally smooth: small, structured changes in the input should produce bounded, predictable changes in the output. To respect this property, our initialization must construct new token embeddings as smooth functions of their constituent subtokens, explicitly preserving their sequential order. This prevents  distributional shifts during early fine-tuning. A formal continuity assumption and order-sensitivity analysis are provided in Appendix \ref{app:continuity_proof} and \ref{app:order-sensitivity}.

\subsection{Spectral Merging Operator}
\label{subsec:merging}
We define a spectral merging operator $\Psi: \bigcup_{m=1}^{M} \mathbb{R}^{m \times d} \rightarrow \mathbb{R}^d$ that maps a variable-length sequence of subtoken embeddings to a single conditioning vector. To preserve sequential structure while satisfying the continuity principle, we treat the embedding matrix $\mathbf{S} = [\mathbf{e}_1; \dots; \mathbf{e}_m] \in \mathbb{R}^{m \times d}$ as a set of $d$ discrete signals and apply a two-stage pipeline.

\paragraph{Fourier Downsampling.}
For each embedding dimension $j \in \{1,\dots,d\}$, we compute the discrete Fourier transform (DFT) $\hat{\mathbf{S}}_j(k)$. In speech representations, low-frequency components capture slow, global phonetic compositionality, whereas high-frequency components often reflect sharp boundary artifacts between subtokens~\cite{rabiner1993fundamentals}. We therefore apply a Gaussian low-pass filter:
\begin{equation}
\tilde{\mathbf{S}}_j(k) = \hat{\mathbf{S}}_j(k) \cdot \exp\left(-\frac{k^2}{2\sigma^2}\right),
\label{eq:gaussian_filter}
\end{equation}
retaining only the first $K_{\mathrm{eff}}$ frequency bins, where
$K_{\mathrm{eff}}=\min(K,\lfloor m/2\rfloor+1)$, to obtain
$\widetilde{S}^{\downarrow}\in\mathbb{C}^{K_{\mathrm{eff}}\times d}$.
In practice, we use $K=4$ as the default maximum number of retained bins. This operation projects the subtoken sequence onto a smooth, bandlimited subspace that naturally suppresses discontinuities at token boundaries.

\paragraph{Linear Readout.}
We reconstruct a conditioning vector by flattening the real and imaginary components of the filtered coefficients into a feature vector $\phi(\mathbf{S}) \in \mathbb{R}^{2Kd}$, followed by a linear projection:
\begin{equation}
\Psi(\mathbf{S}) = \mathbf{W} \phi(\mathbf{S}) + \mathbf{b}, \quad \mathbf{W} \in \mathbb{R}^{d \times 2Kd}, \; \mathbf{b} \in \mathbb{R}^d.
\label{eq:linear_readout}
\end{equation}
Critically, by preserving the complex DFT coefficients, $\Psi$ remains sensitive to phase information, which directly encodes the temporal progression of the subtoken sequence. Unlike mean pooling, which is permutation-invariant, under the conditions of Proposition 1, our operator is order-sensitive: reversing the input order produces a different output embedding. ($\Psi(\mathbf{S}) \neq \Psi(\mathbf{S}^{\text{rev}})$ for non-palindromic sequences). A formal order-sensitivity analysis is provided in Appendix~\ref{app:order-sensitivity}.

\subsection{Closed-Form Self-Calibration}
\label{subsec:calibration}
To align the spectral operator $\Psi$ with the geometry of the pretrained conditioning space $\mathcal{C}$ without requiring target-language audio, we derive the projection parameters $(\mathbf{W}, \mathbf{b})$ analytically. For each token $t \in \mathcal{V}_{\text{old}}$ that admits a valid subtoken decomposition, we construct a self-supervised pair $(\mathbf{S}^{(i)}, \mathbf{c}_{\text{target}}^{(i)} = E(t))$. Decompositions are obtained using the original BPE tokenizer that generated $\mathcal{V}_{\text{old}}$, ensuring all constituent subtokens belong to the pretrained vocabulary. Stacking $N$ such pairs yields a feature matrix $\mathbf{X} \in \mathbb{R}^{N \times 2Kd}$ and target matrix $\mathbf{Y} \in \mathbb{R}^{N \times d}$. We solve the ridge regression objective:
\begin{equation}
    \min_{\mathbf{W}, \mathbf{b}} \left\| \mathbf{X} \mathbf{W}^\top + \mathbf{1}\mathbf{b}^\top - \mathbf{Y} \right\|_F^2 + \lambda \|\mathbf{W}\|_F^2,
    \label{eq:ridge_obj}
\end{equation}
which admits the closed-form solution:
\begin{equation}
    \begin{bmatrix}
        \mathbf{W}^{*\top} \\
        \mathbf{b}^{*\top}
    \end{bmatrix}
    =
    \left( \tilde{\mathbf{X}}^\top \tilde{\mathbf{X}} + \lambda \mathbf{R} \right)^{-1}
    \tilde{\mathbf{X}}^\top \mathbf{Y},
    \label{eq:ridge_sol}
\end{equation}
where $\tilde{\mathbf{X}} = [\mathbf{X} \mid \mathbf{1}]$ augments the features with a bias column, and $\mathbf{R} = \operatorname{diag}(1,\dots,1,0)$ excludes the bias term from $L_2$ regularization. The matrix $\mathbf{W}^*$ and vector $\mathbf{b}^*$ are obtained from the first $2Kd$ rows and the final row of the solution, respectively. This calibration protocol ensures: (i) \textit{manifold alignment} by minimizing projection error onto $\mathcal{C}$; (ii) \textit{order consistency} by learning from linguistically ordered subtoken sequences; and (iii) \textit{continuity preservation} by bounding the conditioning error, which directly limits acoustic deviation via Assumption~\ref{ass:lipschitz_app}.

\paragraph{Non-Latin Script Handling.}
For languages with non-Latin orthographies (e.g., Cyrillic), we first transliterate the target token into an English IPA representation using a standard mapping table. This phonetic string is then segmented into subtokens from $\mathcal{V}_{\text{old}}$, yielding the sequence $\mathbf{s}_1,\dots,\mathbf{s}_m$. At inference, the novel embedding is computed as $\mathbf{c}_{\text{new}} = \Psi^*(\mathbf{S}_{\text{new}}) = \mathbf{W}^* \phi(\mathbf{S}_{\text{new}}) + \mathbf{b}^*$ and subsequently norm-normalized to match the statistical distribution of the pretrained embeddings. Default values for the hyperparameters are $K=4$, $\sigma=2.0$ and $\lambda = 1e^{-6}$ we refer to Appendix.\ref{k_sigma_just} for more details.

\begin{center}
\fbox{\begin{minipage}{0.95\columnwidth}
\small
\setlength{\abovedisplayskip}{3pt}
\setlength{\belowdisplayskip}{3pt}
\setlength{\jot}{2pt}
\textbf{SEAM Algorithm Summary.} Given subtoken embeddings $\mathbf{s}_1,\dots,\mathbf{s}_m \in \mathbb{R}^d$, let $K_{\mathrm{eff}}=\min\left(K,\left\lfloor m/2\right\rfloor+1\right)$:
\[
\begin{aligned}
\hat{\mathbf{S}}(k) &= \sum_{n=1}^{m} \mathbf{s}_n e^{-2\pi i k (n-1)/m}, \\
\hat{\mathbf{S}}_{\text{LP}}(k) &= \hat{\mathbf{S}}(k) \cdot e^{-k^2/(2\sigma^2)}, \quad k=0,\dots,K_{\mathrm{eff}}-1, \\
\phi(\mathbf{s}_{1:m}) &= \big[\operatorname{Re}(\hat{\mathbf{S}}_{\text{LP}}(0)), \operatorname{Im}(\hat{\mathbf{S}}_{\text{LP}}(0)), \ldots, \\
&\quad \operatorname{Re}(\hat{\mathbf{S}}_{\text{LP}}(K_{\mathrm{eff}}-1)), \operatorname{Im}(\hat{\mathbf{S}}_{\text{LP}}(K_{\mathrm{eff}}-1))\big]^\top, \\
\mathbf{c}_{\text{new}} &= \mathbf{W}\,\phi(\mathbf{s}_{1:m}) + \mathbf{b}, \\
\begin{bmatrix}\mathbf{W}^\top\\\mathbf{b}^\top\end{bmatrix} &= \bigl(\tilde{\mathbf{X}}^\top\tilde{\mathbf{X}} + \lambda\mathbf{R}\bigr)^{-1}\tilde{\mathbf{X}}^\top\mathbf{Y}.
\end{aligned}
\]
\end{minipage}}
\end{center}

\paragraph{Implementation summary.}
The SEAM algorithm proceeds in two stages:

\noindent
\textbf{1. Learning the projector.}
From the old vocabulary $\mathcal{V}_{\text{old}}$, we collect self‑supervised training pairs $(S^{(i)}, c_{\text{target}}^{(i)})$ by decomposing each existing token $t$ into its constituent subtokens $\{s_1,\dots,s_m\}$ (all also in $\mathcal{V}_{\text{old}}$) using the original BPE tokenizer, and setting $c_{\text{target}}^{(i)} = E(t)$. For each such pair we compute the spectral feature vector $\phi(S^{(i)})$ and then solve the ridge regression problem (Eq.~\ref{eq:ridge_obj}) to obtain the closed‑form linear projection $W,b$.

\noindent
\textbf{2. Embedding a new token.}
For a novel token $t_{\text{new}}$ (e.g., a Kazakh Cyrillic word), we first map it to an English IPA representation using a standard transliteration table. We then decompose this IPA string into subtokens present in $\mathcal{V}_{\text{old}}$, obtaining $\mathbf{s}_1,\dots,\mathbf{s}_m$. The final embedding is computed as $\mathbf{c}_{\text{new}} = \mathbf{W}\,\phi(\mathbf{s}_1,\dots,\mathbf{s}_m) + \mathbf{b}$ and normalized to match the average norm of the original ones.

\section{Experiments}
\label{sec:experiments}

We evaluate SEAM across three diverse languages Russian (Slavic, Cyrillic), Kazakh (Turkic, Cyrillic/Latin), and French (Latin) to assess its effectiveness in low-resource and cross-lingual adaptation. All primary experiments use Chatterbox~\cite{chatterbox} as the TTS backbone. In this paper we mainly focus on three core evaluation dimensions: 
(1) \textbf{sample efficiency} under fixed training budgets (5h, 25h, and 100h), quantified via automatic speech recognition metrics (CER); and 
(2) \textbf{perceptual quality}, validated through a native-speaker Side-By-Side comparison (SBS) study.
(3)  \textbf{New Vocabulary size}
investigate how other methods perform in comparison with SEAM when the number of new added tokens increases.
Experiments with a different architecture, VibeVoice~\cite{vibevoice}, and dataset-split robustness are provided in Appendix~\ref{app:further_experiments}.
\subsection{Experimental Setup}
\label{subsec:setup}

\paragraph{Evaluation metrics.} 
We evaluate model performance using the Character Error Rate (CER), computed via pretrained automatic speech recognition (ASR) models. Specifically, we employ \textit{nvidia/stt\_kk\_ru\_fastconformer\_hybrid\_large}, \textit{nvidia/stt\_fr\_fastconformer\_hybrid\_large\_pc} for Russian, Kazakh and French; refer to Appendix~\ref{app:asr_baseline} for a detailed analysis of these ASR baselines. All reported results are averaged across three independent random seeds. Evaluation is performed on a fixed set of 100 audio samples drawn from \textit{rulibrispeech}, \textit{KazakhTTS}, and \textit{AdrienB134/french-unique-speaker-tts}.

\paragraph{Baselines.}
We compare against:
1) \textbf{Random}: $c_t \sim \mathcal{N}(0, \sigma^2 I)$ with $\sigma$ matched to pretrained embedding norm. \\
 \textbf{Mean}: $c_t = \frac{1}{m} \sum_{i=1}^m E(s_i)$, the arithmetic mean of subtoken embeddings.

All methods are fine-tuned under identical optimization schedules. Implementation details and additional metrics are provided in Appendix~\ref{app:implementation}.
\paragraph{Transliteration} We transliterate non-Latin scripts using the transliterate library (English/Russian), qazaq\_transliterator (Kazakh), and standard text normalization (French).

\subsection{Sample Efficiency Under Fixed Budgets}
\label{subsec:sample_efficiency}

Table~\ref{tab:fixed_budget} reports CER after training on 5h, 25h and 100h budgets. SEAM provides its largest gains in low-resource regimes. At 25h, SEAM improves over mean initialization by 2.22 CER points on Russian, 10.63 points on Kazakh, and 0.87 points on French. Against random initialization, the corresponding improvements are 9.16, 27.34, and 3.77. Notably, SEAM approaches or matches the performance of baselines trained on 100h, corresponding to a 4× reduction in required data. At 100h, all methods converge, with differences falling within ASR evaluation variance, confirming that SEAM's primary advantage is sample efficiency rather than asymptotic performance.
It is important to mention that the increase in CER from 5h to 25h for Random/Mean on Kazakh is attributable to word-level overlap variance across data splits in this extremely low-resource regime
\begin{table}[t]
\centering
\caption{CER (\%) after training on fixed data budgets. Lower is better. SEAM achieves comparable performance to baselines with up to 4$\times$ less data for some languages.}
\label{tab:fixed_budget}
\small
\resizebox{0.7\columnwidth}{!}{
\begin{tabular}{lcccccc}
\toprule
\textbf{Language} & \textbf{Budget} & \textbf{Init} & \textbf{CER} & \textbf{WER} \\
\midrule
\multirow{9}{*}{Russian}
     & \multirow{3}{*}{5 h} 
        & Random & 20.8 & 42.49 \\
    &   & Mean   & 17.53 & 36.73 \\
    &   & \textbf{SEAM} &\textbf{14.88} & \textbf{34.41} \\
    \cmidrule{2-5}
    & \multirow{3}{*}{25 h} 
        & Random & 17.66 & 27.18 \\
    &   & Mean   & 10.72 & 24.58 \\
    &   & \textbf{SEAM} & \textbf{8.50} & \textbf{20.01} \\
    \cmidrule{2-5}
    & \multirow{3}{*}{100 h} 
        & Random & \textbf{7.62} & 18.40 \\
    &   & Mean   & 9.64 & 19.87 \\
    &   & \textbf{SEAM} & 7.68 & \textbf{16.92} \\
\midrule
\multirow{9}{*}{Kazakh} 
    & \multirow{3}{*}{5 h} 
            & Random & 11.62 & 26.3 \\
        &   & Mean   & 8.55 & 20.9 \\
        &   & \textbf{SEAM} &\textbf{4.99} & \textbf{15.12} \\
        \cmidrule{2-5}
    & \multirow{3}{*}{25 h} 
            & Random & 32.60 & 39.89 \\
        &   & Mean   & 15.89 & 21.60 \\
        &   & \textbf{SEAM} &\textbf{5.26} & \textbf{13.19} \\
        \cmidrule{2-5}

    & \multirow{3}{*}{100 h} 
        & Random & 9.54 & 16.90 \\
    &   & Mean   & 8.50 & 16.02 \\
    &   & \textbf{SEAM} &\textbf{4.39} & \textbf{11.11} \\
   
\midrule
\multirow{9}{*}{French} 
    & \multirow{3}{*}{5 h} 
            & Random & 19.07 & 34.27 \\
        &   & Mean   & 18.28 & 34.61 \\
        &   & \textbf{SEAM} &\textbf{17.86} & \textbf{31.7} \\
        \cmidrule{2-5}
    & \multirow{3}{*}{25 h} 
        & Random & 15.70 & 29.68 \\
    &   & Mean   & 12.80 & 25.82 \\
    &   & \textbf{SEAM} & \textbf{11.93} & \textbf{23.20} \\
    \cmidrule{2-5}
    & \multirow{3}{*}{100 h} 
        & Random & 16.93 & 23.85 \\
    &   & Mean   & \textbf{9.06} & \textbf{17.20} \\
    &   & \textbf{SEAM} & 9.26 & 17.59 \\
\bottomrule
\end{tabular}
}
\end{table}

\subsection{Human Evaluation}
\label{subsec:human}

We conduct a blind Side-by-Side (SBS) preference study with 5 native speakers per language. Each annotator evaluates 25 synthesized utterances, comparing audio from two conditions:
(1) \textbf{Model A (Ours):} SEAM-initialized embeddings, fine-tuned on \textbf{25h} of speech.
(2) \textbf{Model B (Baseline):} mean initialization, fine-tuned on \textbf{100h} of speech.

For each pair, annotators select their preferred sample based on overall quality (naturalness + intelligibility), or indicate a tie.  Evaluation criteria and annotator instructions are provided in Appendix~\ref{app:eval_guidelines}.
Figure~\ref{fig:sbs_bars} visualizes the SBS preference distribution. Model A (SEAM) against mean initialization that was trained on 4x more data across three languages.
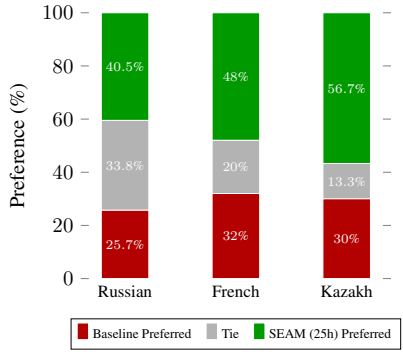
\begin{figure}[t]
\centering

\resizebox{0.7\columnwidth}{!}{
\begin{tikzpicture}
\begin{axis}[
    ybar stacked,
    width=0.95\columnwidth,
    height=6.5cm,
    bar width=25pt,
    nodes near coords={\pgfmathprintnumber{\pgfplotspointmeta}\%},
    nodes near coords style={
        font=\scriptsize,
        color=white,
        text opacity=1,
        inner sep=1pt,
        /pgf/number format/fixed,
        /pgf/number format/precision=1
    },
    ymin=0,
    ymax=100,
    ylabel={Preference (\%)},
    symbolic x coords={Russian, French, Kazakh},
    xtick=data,
    xticklabel style={font=\small},
    legend style={
        at={(0.5,-0.15)},
        anchor=north,
        legend columns=3,
        font=\scriptsize,
        /tikz/every even column/.append style={column sep=0.5em}
    },
    legend cell align=left,
    enlarge x limits=0.2,
    grid=none,
    axis line style={opacity=0},
    every axis plot post/.append style={draw=white, line width=0.4pt}
]

\addplot+[fill=red!70!black] coordinates {
    (Russian, 25.7)
    (French, 32.0)
    (Kazakh, 30.0)
};
\addlegendentry{Baseline Preferred}

\addplot+[fill=gray!60] coordinates {
    (Russian, 33.8)
    (French, 20.0)
    (Kazakh, 13.3)
};
\addlegendentry{Tie}

\addplot+[fill=green!60!black] coordinates {
    (Russian, 40.5)
    (French, 48.0)
    (Kazakh, 56.7)
};
\addlegendentry{SEAM (25h) Preferred}

\end{axis}
\end{tikzpicture}
}
\caption{Side-by-Side human preference results: SEAM-initialized model (25h training) vs. mean-pooling baseline (100h training). Green indicates preference for SEAM despite using 4$\times$ less data. Results are shown for Russian, French, and Kazakh.}
\label{fig:sbs_bars}
\end{figure} 
\noindent We refer to Appendix \ref{statisticalrel} for detailed information about the SBS study.

\subsection{Sensitivity to Vocabulary Expansion Size}
\label{subsec:ablation_tokens}

We investigate the sensitivity of our method to vocabulary expansion size by varying the number of newly added tokens ($|\mathcal{V}_{\text{new}}| \in \{200, 400, 600, 800, 1000\}$) and evaluating CER after fine-tuning on 25 hours of data. As shown in Figure~\ref{fig:ablation_tokens}, SEAM scales gracefully, exhibiting only marginal performance degradation as $|\mathcal{V}_{\text{new}}|$ increases. In contrast, both random and average initialization suffer sharp deterioration beyond 500 tokens, demonstrating that SEAM produces robust embeddings even under substantial vocabulary expansion. While we constrained $|\mathcal{V}_{\text{new}}| = 200$ in our main experiments to ensure a competitive comparison, larger vocabularies inherently yield shorter token sequences and faster inference. Notably, SEAM's relative advantage becomes increasingly pronounced at higher expansion sizes, highlighting its strong scalability and efficiency in practical deployment scenarios.
\begin{figure}[h]
\vspace{-0.59cm}
    \centering
    \includegraphics[width=0.85\columnwidth]{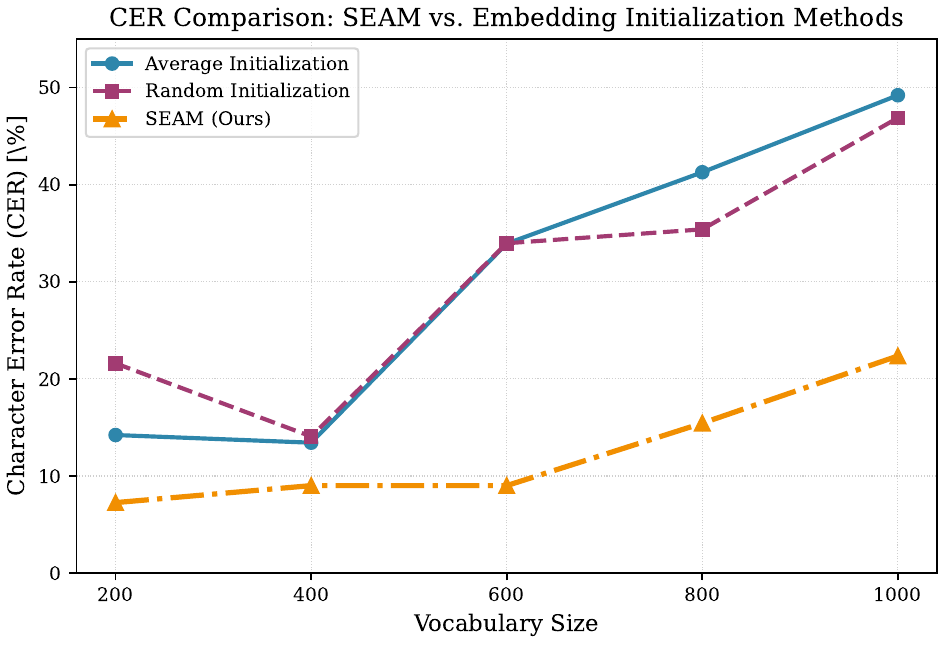}
    \caption{CER versus vocabulary expansion size. SEAM maintains robust performance as vocabulary size increases, while mean and random initialization degrade}
\label{fig:ablation_tokens}
\vspace{-0.5cm}
\end{figure}

\subsection{Extended Baseline Comparison}
\label{app:extended_baselines}

We compare SEAM with three additional order-aware initialization strategies on Russian with a 5-hour budget under identical optimization schedules.
\begin{table}[h]
\centering
\small
\caption{Extended baseline comparison on Russian (5h fine-tuning). Lower is better.}
\label{tab:extended_baselines}
\begin{tabular}{lcc}
\toprule
\textbf{Initialization Method} & \textbf{CER (\%)} & \textbf{WER (\%)} \\
\midrule
Random & 20.80 & 42.49 \\
Mean Pooling & 17.53 & 36.73 \\
Position-Aware Pooling & 23.37 & 54.70 \\
Attention Pooling & 20.14 & 50.03 \\
Learned Attention & 15.30 & 42.86 \\
\textbf{SEAM (Ours)} & \textbf{14.88} & \textbf{34.41} \\
\bottomrule
\end{tabular}
\end{table}
\textbf{Position-Aware Pooling} extracts permutation-variant statistics (mean, max, first, last, and positionally weighted sum) and projects them via closed-form ridge regression with $\lambda = 10^{-6}$.\\
\textbf{Attention Pooling} applies fixed random query vectors for scaled dot-product attention over subtoken embeddings, followed by ridge projection, with no learned parameters.
\textbf{Learned Attention} trains a single-head attention layer via AdamW (lr $= 5 \times 10^{-4}$) on the same self-supervised pairs used by SEAM, and is the strongest order-aware baseline.
SEAM outperforms all baselines because its spectral low-pass filter suppresses high-frequency tokenizer-boundary artifacts, explaining the poor results of Position-Aware Pooling (CER 23.37\%) and Attention Pooling (CER 20.14\%), which treat all frequency components equally. Learned Attention is competitive (CER 15.30\%) but converges poorly under the limited calibration data, yielding approximately 10\% lower cosine similarity to ground-truth embeddings and higher MSE than SEAM's closed-form solution. This confirms that gradient-based attention requires substantially more data than SEAM's analytic ridge regression in the extreme low-resource initialization setting.

\subsection{Ablation: Order Sensitivity}
\label{app:order_ablation}

To test whether temporal order preservation drives SEAM's gains, we disrupt subtoken order before spectral filtering on Russian with a 5-hour fine-tuning budget.

\begin{table}[h]
\centering
\small
\caption{Order sensitivity ablation on Russian (5h).}
\label{tab:order_ablation}
\begin{tabular}{lcc}
\toprule
\textbf{Method} & \textbf{CER (\%)} & \textbf{WER (\%)} \\
\midrule
SEAM (Original) & \textbf{14.88} & \textbf{34.41} \\
SEAM (Reversed) & 19.67 & 51.99 \\
SEAM (Shuffled) & 22.91 & 53.66 \\
\bottomrule
\end{tabular}
\end{table}

\noindent\textbf{Reversed} strictly inverts the subtoken sequence ($s_1,\dots,s_m \rightarrow s_m,\dots,s_1$), while \textbf{Shuffled} randomly permutes it (seed=42), completely destroying temporal order. In both cases, spectral filtering and ridge regression are applied identically to the original SEAM pipeline.

Reversal increases CER by 4.79 absolute points (14.88\% $\rightarrow$ 19.67\%), while shuffling increases it by 8.03 points (14.88\% $\rightarrow$ 22.91\%). Shuffling also nearly doubles WER (34.41\% $\rightarrow$ 53.66\%). Since shuffling is more damaging than reversal, SEAM relies not only on sequence directionality but on the precise temporal phonetic trajectory encoded in the DFT phase. This confirms that order-preserving spectral merging, rather than ridge alignment alone, is the primary source of SEAM's advantage.
\section{Conclusion}
We presented SEAM, a zero-audio initialization method for low-resource TTS vocabulary expansion. By spectrally merging subtoken embeddings and aligning them via closed-form regression, SEAM preserves phonetic order and reduces sample complexity by up to 4× across languages and architectures. While dependent on subtoken decomposability and IPA transliteration, SEAM offers a lightweight, reproducible pathway to stabilize early-stage TTS adaptation.
\paragraph{\textbf{Acknowledgments}}
This research was supported as part of the ITMO University Research Projects in AI Initiative.
\section{Limitations}
Despite its data efficiency, SEAM operates within three primary constraints:
\begin{itemize}
    \item \textbf{Subtoken \& IPA Dependency:} Initialization requires reliable decomposition into known subtokens. For non-Latin scripts or languages lacking robust IPA transliteration tables, mapping errors may propagate to the conditioning space.
    \item \textbf{Fine-Tuning Requirement:} SEAM addresses the initialization bottleneck but does not replace downstream adaptation; acoustic alignment still requires paired target-language audio but much less than other initialization methods.
    \item \textbf{Evaluation Variance in Extreme Low-Resource Regimes}: At very small budgets, word-level overlap between training and test splits can inflate metrics non-monotonically, as observed for Kazakh. 
\end{itemize}

\section{Ethical Considerations}
SEAM aims to reduce barriers for low-resource language TTS, promoting equitable access to speech technology. However, like all voice synthesis methods, it carries dual-use risks: malicious actors could exploit improved initialization for unauthorized voice cloning or disinformation. We mitigate this by releasing only embedding initialization code not full model weights and requiring explicit consent for any deployment involving human voices. All experiments use publicly available, ethically sourced corpora. We encourage adopters to follow local regulations and community guidelines, particularly when working with  under-documented languages.


\bibliography{custom}
\appendix

\newpage
\section{Theoretical Overview}

\subsection{Continuity Principle and Acoustic Stability}
\label{app:continuity_proof}

Let $\mathcal{C} \subset \mathbb{R}^d$ denote the pretrained conditioning space and $\mathcal{X} \subset \mathbb{R}^D$ the acoustic output space. The TTS generator defines a stochastic mapping $\mathbf{x} = \Phi_\theta(\mathbf{c}, \mathbf{z})$, where $\mathbf{c} \in \mathcal{C}$ is a conditioning vector and $\mathbf{z} \sim \mathcal{N}(\mathbf{0}, \mathbf{I})$ represents diffusion or autoregressive noise. Under standard regularity conditions for modern generative audio models, the conditional expectation operator is Lipschitz continuous~\cite{chen2023likelihood}:

\begin{assumption}[Lipschitz Continuity]
\label{ass:lipschitz_app}
There exists a constant $L > 0$ such that for all $\mathbf{c}_a, \mathbf{c}_b \in \mathcal{C}$:
\begin{equation}
\big\| \mathbb{E}_{\mathbf{z}}[\Phi_\theta(\mathbf{c}_a, \mathbf{z})] - \mathbb{E}_{\mathbf{z}}[\Phi_\theta(\mathbf{c}_b, \mathbf{z})] \big\|_2 \leq L \|\mathbf{c}_a - \mathbf{c}_b\|_2.
\end{equation}
\end{assumption}

\noindent\textbf{Intuition for TTS.} Assumption~\ref{ass:lipschitz_app} formalizes a fundamental property of stable generative models: the acoustic output space is metrically aligned with the conditioning manifold. Consequently, large or erratic perturbations to a token embedding can cause disproportionate acoustic artifacts, distributional shift, or gradient instability during fine-tuning. To mitigate this, the embedding $\mathbf{c}_t$ for a novel token must be constructed as a \textit{smooth}, \textit{structure-preserving} function of its subtoken embeddings $\{\mathbf{e}_i = E(s_i)\}_{i=1}^m$. By enforcing temporal order and spectral smoothness (Section~\ref{subsec:merging}), SEAM ensures that $\|\mathbf{c}_t - \mathbf{c}_{\text{ref}}\|_2$ remains small for phonetically similar tokens, thereby bounding the expected acoustic deviation during early adaptation stages and stabilizing optimization from step zero.

\subsection{Order-Sensitivity Analysis}
\label{app:order-sensitivity}
\begin{proposition}[Order sensitivity; sufficient conditions]
\label{prop:order_sensitivity}
Let
\[
S = [s_1; \ldots; s_m] \in \mathbb{R}^{m \times d},
\qquad m \ge 2,
\]
and let
\[
S^{\mathrm{rev}} = [s_m; \ldots; s_1]
\]
denote the linearly reversed sequence. Let
\begin{equation}
K_{\mathrm{eff}}
=
\min\left(
K,
\left\lfloor \frac{m}{2} \right\rfloor + 1
\right)
\end{equation}
denote the number of unique retained DFT bins. Suppose that:

\begin{enumerate}
    \item $K_{\mathrm{eff}} \ge 2$;
    \item at least one retained non-DC bin is reversal-asymmetric, i.e., there exist
    $k \in \{1,\ldots,K_{\mathrm{eff}}-1\}$ and embedding dimension $j$ such that
    \begin{equation}
    e^{-2\pi i k (m-1)/m}\,\hat{S}_j[k]^*
    \neq
    \hat{S}_j[k];
    \end{equation}
    \item the projection matrix $W$ does not annihilate the corresponding feature difference.
\end{enumerate}

Then
\begin{equation}
\Psi(S^{\mathrm{rev}})
\neq
\Psi(S).
\end{equation}
\end{proposition}

\begin{proof}
For each embedding dimension $j$, define the DFT coefficient
\begin{equation}
\hat{S}_j[k]
=
\sum_{n=0}^{m-1}
S_j[n]\,
e^{-2\pi i k n/m}.
\end{equation}

For the linearly reversed sequence
$
S^{\mathrm{rev}}[n] = S[m-1-n],
$
the DFT satisfies
\begin{align}
\hat{S}^{\mathrm{rev}}_j[k]
&=
\sum_{n=0}^{m-1}
S_j[m-1-n]\,
e^{-2\pi i k n/m}
\nonumber \\
&=
e^{-2\pi i k (m-1)/m}
\sum_{r=0}^{m-1}
S_j[r]\,
e^{2\pi i k r/m}
\nonumber \\
&=
e^{-2\pi i k (m-1)/m}
\hat{S}_j[-k]
\nonumber \\
&=
e^{-2\pi i k (m-1)/m}
\hat{S}_j[k]^*,
\end{align}
where the last equality follows because each embedding dimension is real-valued.

Thus, reversal changes the complex phase structure of the retained spectrum. For real-valued embedding dimensions, this generally affects the imaginary components retained by $\phi(S)$. If at least one retained non-DC frequency component is reversal-asymmetric, then
\begin{equation}
\phi(S^{\mathrm{rev}})
\neq
\phi(S).
\end{equation}

If, in addition, $W$ does not annihilate this difference, i.e.,
\begin{equation}
W\left(
\phi(S^{\mathrm{rev}}) - \phi(S)
\right)
\neq
0,
\end{equation}
then
\begin{equation}
\Psi(S^{\mathrm{rev}})
=
W\phi(S^{\mathrm{rev}}) + b
\neq
W\phi(S) + b
=
\Psi(S).
\end{equation}
\end{proof}
This is a sufficient-condition statement rather than an unconditional guarantee. If all
retained spectral coefficients are reversal-invariant, for example when only the DC
component or other perfectly symmetric components are retained, the operator may map
$S$ and $S^{\mathrm{rev}}$ to the same conditioning vector. The proposition therefore
requires asymmetry in at least one retained non-DC frequency bin.

\paragraph{Choice of $K$ and $\sigma$.}
\label{k_sigma_just}
We set $K=4$ as the default maximum number of retained frequency bins and $\sigma=2.0$.
For short subtoken sequences, $K$ should be interpreted as an upper bound: the number of
unique retained DFT coefficients is effectively
\[
K_{\mathrm{eff}}
=
\min\left(K,\left\lfloor \frac{m}{2}\right\rfloor+1\right).
\]
The low-pass argument is strongest when the sequence length permits
$K_{\mathrm{eff}}\ll m/2$, so that the retained coefficients remain well below the
Nyquist region where boundary artifacts dominate. For very short decompositions, the
sufficient conditions of Proposition~\ref{prop:order_sensitivity} may not hold, and the
order-sensitivity claim is instead supported empirically by the order-perturbation
ablation in Table~\ref{tab:order_ablation}.

Empirical spectral analysis (Figure~\ref{fig:spectral_analysis}) shows that cumulative
energy exceeds 99\% by $k=3$, indicating that the subtoken embedding trajectory is
effectively bandlimited. Retaining bins $k\in\{0,1,2,3\}$ therefore preserves the
dominant low-frequency phonetic structure while strongly attenuating high-frequency
boundary artifacts. The Gaussian gain with $\sigma=2.0$ satisfies
\[
G(k)=\exp\left(-\frac{k^2}{2\sigma^2}\right),
\]
giving
\[
G(0)=1.0,\; G(1)\approx 0.88,\; \]
\[G(2)\approx 0.61,\; G(3)\approx 0.32,\; G(4)\approx 0.14.
\]
Thus, high-frequency components are suppressed by more than $85\%$ by $k=4$.

\begin{figure}[h]
    \centering
\includegraphics[width=1.0\columnwidth]{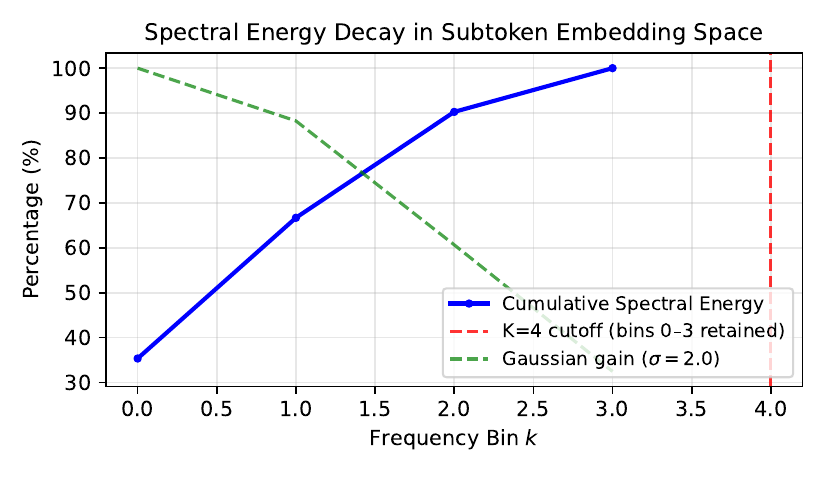}
    \caption{Spectral energy decay in subtoken embedding sequences. Cumulative energy (blue) exceeds 99\% by $k=3$, confirming bandlimited compositionality. The Gaussian filter $\sigma=2.0$ (green) suppresses high-frequency noise while retaining bins $k \in \{0,1,2,3\}$ (red dashed line). This empirically validates $(K=4, \sigma=2.0)$ without hyperparameter tuning.}
    \label{fig:spectral_analysis}
\end{figure}

\section{Further Experiments}
\label{app:further_experiments}
\subsection{Implementation details.}
\label{app:implementation}

All experiments were conducted on a single workstation equipped with an NVIDIA RTX 3090 GPU (24 GB VRAM) and an Intel Core i9-14900K CPU (24 cores, 32 logical processors, 3.2 GHz base clock). Training pipelines were implemented using the Hugging Face \texttt{transformers} and \texttt{accelerate} ecosystems with mixed-precision (\texttt{bf16}) enabled.

\paragraph{Chatterbox (0.5B) Fine-Tuning.} 
We perform full-parameter fine-tuning of the English-only Chatterbox base model. Optimization uses AdamW with a learning rate of $5 \times 10^{-4}$, 10\% warmup steps, and a cosine decay schedule. The per-device batch size is 8 with 8 gradient accumulation steps (effective batch size 64) over 1 epoch. Data loading utilizes 8 workers with pinned memory disabled to optimize CPU--GPU transfer. All automatic metrics are computed post-hoc on the held-out test set.

\paragraph{VibeVoice (1.5B) Fine-Tuning.} 
For cross-architecture validation, we adapt VibeVoice using Low-Rank Adaptation (LoRA). We apply LoRA to all linear projection layers (\texttt{q\_proj}, \texttt{k\_proj}, \texttt{v\_proj}, \texttt{o\_proj}, \texttt{gate\_proj}, \texttt{up\_proj}, \texttt{down\_proj}) with rank $r=128$ and scaling factor $\alpha=256$. The diffusion head is frozen , and loss weights are set to 1.4 for the diffusion objective and 0.04 for cross-entropy. Training uses a learning rate of $5 \times 10^{-5}$, a cosine scheduler with 3\% warmup, maximum gradient norm of 0.8, and voice prompt dropout of 0.2. The effective batch size is 16 (4 per device $\times$ 4 accumulation steps) over 3 epochs, with a DDPM set to 10 steps to stabilize diffusion sampling.

All runs use fixed random seeds for reproducibility, and hyperparameters are kept constant across languages to ensure fair comparison.

\paragraph{Ridge Regularization.}
The ridge regression regularization coefficient in Eq.~3 is set to $\lambda = 10^{-6}$ for all experiments across all languages and architectures. This value provides minimal regularization sufficient to ensure numerical stability of the closed-form solution (Eq.~4) without biasing the projection away from the least-squares optimum. No language-specific or architecture-specific tuning of $\lambda$ was performed.

\paragraph{IPA Transliteration Setup.}
\label{app:ipa_setup}
For languages with non-Latin orthographies, SEAM requires mapping target-language tokens to an English IPA representation before subtoken decomposition. We use the following tools:
\begin{itemize}
    \item \textbf{Russian}: The \texttt{transliterate} Python library maps Cyrillic graphemes to their English IPA equivalents. Standard Russian phonological rules (vowel reduction, consonant palatalization) are applied during normalization.
    \item \textbf{Kazakh}: The \texttt{qazaq\_transliterator} library handles Cyrillic-to-Latin conversion, followed by IPA mapping. Kazakh-specific vowel harmony and consonant assimilation rules are respected during transliteration.
    \item \textbf{French}: Standard orthographic-to-phonemic normalization is applied, handling liaison, elision, and nasal vowel conventions.
\end{itemize}
The resulting IPA strings are segmented using the original BPE tokenizer that generated $\mathcal{V}_{\text{old}}$, ensuring all constituent subtokens belong to the pretrained vocabulary.

\subsection{ASR Evaluation Baseline}
\label{app:asr_baseline}

To establish the error floor of our automatic evaluation pipeline, we evaluated the pretrained ASR models on the original, human-recorded audio from each test set (i.e., ground-truth speech, not synthesized). Table~\ref{tab:asr_baseline} reports the resulting Character and Word Error Rates. These values represent the irreducible lower bound for CER/WER in our TTS evaluations: synthesized output is unlikely to achieve a lower error rate than the ASR model’s performance on clean, native speech.

\begin{table}[h]
\centering
\small
\caption{ASR baseline performance on original human-recorded audio. Lower is better. These values define the error floor for TTS evaluation.}
\label{tab:asr_baseline}
\begin{tabular}{lcc}
\toprule
\textbf{Language} & \textbf{CER (\%)} & \textbf{WER (\%)} \\
\midrule
Russian & 4.88 & 14.68 \\
French & 0.63 & 3.30 \\
Kazakh & 0.98 & 5.68 \\
\bottomrule
\end{tabular}
\end{table}

These results indicate that the ASR models exhibit high accuracy on French and Kazakh (CER $<1\%$), while Russian presents a more challenging evaluation setting (CER $\approx 5\%$). Consequently, when interpreting the TTS results in Table~1, the effective TTS-induced error for Russian should be understood relative to this 4.88\% baseline. For instance, SEAM's 7.68\% CER at 100h corresponds to an approximate TTS-specific error of $7.68 - 4.88 = 2.80\%$, confirming that the synthesized speech remains close to human-level intelligibility. This baseline analysis ensures that reported improvements reflect genuine gains in synthesis quality rather than artifacts of ASR model variance.

\subsection{Evaluation metrics.}
In this research, we reported:
\begin{itemize}
    \item \textbf{Objective}: Character Error Rate (CER) using  pretrained ASR models.
    \item \textbf{Efficiency}:   training data count with the corresponding  CER at each count.
    \item \textbf{Human}: Side By Side (SBS)  (see Section~\ref{subsec:human}).
\end{itemize}

\subsection{LLM Usage}

We acknowledge the use of a large language model (LLM) for editing aid during manuscript preparation. The LLM was employed exclusively for language polishing, including grammar correction, syntactic refinement, and improving the clarity and formatting of figures and visualizations. 
\subsection{Ablation: Sensitivity to Dataset Splits}

To evaluate SEAM's sensitivity to training data composition, we conduct an ablation study on the Kazakh corpus using three independent random splits. Each split preserves the same dataset size, speaker diversity, and phonemic coverage while varying the exact utterance selection. All models are fine-tuned under identical optimization schedules and evaluated on a fixed held-out test set. Performance remains highly consistent across partitions: the WER ranges from 14.97\% to 16.51\% (mean: 15.61\%, $\sigma = 0.66$) and the CER ranges from 4.79\% to 5.67\% (mean: 5.27\%, $\sigma = 0.36$). The low standard deviations ($\sigma^2_{\text{WER}} = 0.43$, $\sigma^2_{\text{CER}} = 0.13$) confirm that SEAM's performance is robust to partition-level variations in the training data.
The low standard deviation across splits indicates that SEAM's initialization is robust to dataset composition and does not rely on  high-quality subsets.

\paragraph{Non-Monotonic Kazakh Trend (5h $\rightarrow$ 25h).}
\label{par:kazakh_nonmonotonic}
Table~1 reveals a counter-intuitive increase in CER from the 5h to the 25h budget for Random and Mean initialization on Kazakh (e.g., Random: 11.62\% $\rightarrow$ 32.60\%). We clarify that this non-monotonic behavior is an artifact of word-level (token-level) lexical overlap between the training and test splits, rather than a failure of model capacity. Although we enforce strict utterance-level disjointness between training and test sets, in extremely low-resource regimes the specific utterances randomly sampled for the 5h subset may contain a higher density of words that also appear in the 100-utterance test set, artificially inflating performance metrics. Conversely, the 25h subset introduces greater lexical diversity and reduced overlap with the test set, making evaluation inherently harder. Critically, this pattern is consistent across \textit{all} initialization methods, SEAM also exhibits a modest increase (4.99\% $\rightarrow$ 5.26\%) confirming that the trend is driven by data composition rather than any method-specific deficiency. At 100h, all methods stabilize, and SEAM achieves its lowest Kazakh CER of 4.39\%, well above the ASR error floor of 0.98\% (Table~\ref{tab:asr_baseline}).

\subsection{Ablation: Cross-Architecture Generalization}
\label{subsec:ablation_vibevoice}

To assess the architectural agnosticism of our approach, we apply SEAM to VibeVoice, a decoder-based speech model. We fine-tune VibeVoice on a 5-hour Russian corpus using three embedding initialization strategies. As shown in Table~\ref{tab:russian_5h_cer}, SEAM achieves a CER of 6.14\%, significantly outperforming both average (8.34\%) and random (9.26\%) initialization. This result confirms that SEAM's embedding quality generalizes beyond the primary architecture, maintaining a consistent performance advantage under low-resource fine-tuning conditions.
\begin{table}[h]
\centering
\caption{CER (\%) for VibeVoice fine-tuned on Russian (5 hours) under different embedding initialization strategies. Lower is better.}
\label{tab:russian_5h_cer}
\resizebox{0.7\columnwidth}{!}{%
\begin{tabular}{lccc}
\toprule
\textbf{Language} & \textbf{Initialization} & \textbf{Hours} & \textbf{CER (\%)} \\
\midrule
Russian & Random & 5 & 9.26 \\
Russian & Average & 5 & 8.34 \\
Russian & \textbf{SEAM (Ours)} & 5 & \textbf{6.14} \\
\bottomrule
\end{tabular}%
}
\end{table}



\subsection{Ablation: Audio-Guided Initialization}
\label{app:audio_guided}

To assess whether incorporating target-language audio during initialization improves upon SEAM's zero-audio design, we conduct an ablation comparing zero-audio SEAM against an audio-guided variant. We note that, to the best of our knowledge, no prior published work proposes an independent audio-guided initialization method specifically for vocabulary expansion in autoregressive transformer-based TTS. We therefore position this experiment as an ablation validating the zero-audio property, rather than a comparison against an existing baseline.

The \textbf{Audio-Guided} variant proceeds in two stages: (1)~SEAM initialization is applied identically to the zero-audio setting; (2)~the newly initialized embeddings and the LM head are further refined via AdamW optimization on 1 hour of calibration audio from the target language before the main fine-tuning phase begins.

\paragraph{Results.}
Table~\ref{tab:audio_guided} reports the training loss dynamics for both variants on Russian (5h budget).

\begin{table}[h]
\centering
\small
\caption{Training loss dynamics: Zero-audio SEAM vs.\ Audio-Guided variant (Russian, 5h). The audio-guided variant exhibits severe optimization instability.}
\label{tab:audio_guided}
\begin{tabular}{lcc}
\toprule
\textbf{Variant} & \textbf{Initial Loss} & \textbf{Final Loss} \\
\midrule
SEAM (Zero-Audio) & 10.2 & $\sim$4.0 \\
Audio-Guided & 7.2 & $\sim$10.0 \\
\bottomrule
\end{tabular}
\end{table}

While audio-guided refinement achieves a lower initial training loss (7.2 vs.\ 10.2), it introduces severe optimization instability. The refinement step causes embedding norms to inflate to atypical values, distorting the geometric alignment established by SEAM's closed-form calibration. When the main fine-tuning phase begins from this checkpoint, the loss spikes from 7.2 to 15.0 before slowly recovering to approximately 10.0 still substantially worse than zero-audio SEAM's converged loss of $\sim$4.0. In contrast, zero-audio SEAM converges smoothly and monotonically from 10.2 to $\sim$4.0.

This result demonstrates that SEAM's zero-audio initialization is not merely a convenience but a deliberate design choice. The closed-form ridge regression produces embeddings that are geometrically well-calibrated to the pretrained manifold; subsequent audio-guided refinement disrupts this calibration, causing the model to fall into a poor local minimum that requires significantly more data to escape. The absence of any independent audio-based initialization baseline in the literature further underscores that SEAM's zero-audio property fills a genuine methodological gap.

\section{Side-by-Side (SBS) Human Evaluation Guidelines}
\label{app:eval_guidelines}

Annotators evaluated synthesized audio samples using a Side-by-Side (SBS) preference framework, focusing on two orthogonal dimensions: \textbf{Realism} (acoustic naturalness) and \textbf{Linguistic Correctness} (fidelity to orthographic and phonological norms of the target language). Below are the standardized criteria and decision protocols provided to annotators.

\subsection{Dimension 1: Realism (Naturalness)}
Annotators rated each sample on a 5-point scale. The following anchors guided the assessment:
\begin{itemize}
    \item \textbf{5 (Excellent):} Indistinguishable from professional human speech. Prosody is fluid, timbre is natural, and paralinguistic cues (e.g., breath, micro-pauses) are appropriately integrated.
    \item \textbf{4 (Good):} Largely natural but exhibits minor artifacts, such as slight pitch irregularities, mild robotic timbre, or subtle prosodic mismatches (e.g., flattened intonation on interrogatives). Comprehension remains effortless.
    \item \textbf{3 (Fair):} Noticeable synthesis artifacts or pacing irregularities (e.g., unnatural pauses, fragmented rhythm), but speech remains fully intelligible. Pronunciation may occasionally mismatch syntactic context.
    \item \textbf{2 (Poor):} Severe articulation errors, phoneme hallucinations, or repetitive stuttering. Speech is strained or partially unintelligible, significantly impeding comprehension.
    \item \textbf{1 (Bad):} Highly robotic, garbled, or completely unintelligible. Contains nonsensical syllable sequences or arbitrary breaks that destroy syntactic coherence.
\end{itemize}

\subsection{Dimension 2: Linguistic Correctness}
Linguistic correctness measures adherence to the phonological, morphological, and orthographic conventions of the target language (French, Kazakh, or Russian). Annotators identified and penalized the following error categories:

\begin{itemize}
    \item \textbf{Phoneme \& Stress Accuracy:}
    \begin{itemize}
        \item \textit{Russian/Kazakh:} Correct realization of vowel reduction (\textit{akanye}), palatalization (soft vs. hard consonants), and lexical stress placement. Misplaced stress that changes word meaning is a critical error.
        \item \textit{French:} Correct liaison, elision, and nasal vowel production. Failure to apply obligatory liaison (e.g., \textit{les\_amis}) or incorrect nasalization (e.g., \textit{bon} vs. \textit{bonne}) is penalized.
        \item \textit{Kazakh:} Accurate pronunciation of vowel harmony and consonant assimilation rules. Errors in front/back vowel alignment within a word constitute a linguistic error.
    \end{itemize}
    
    \item \textbf{Lexical \& Orthographic Fidelity:}
    \begin{itemize}
        \item Any substitution, omission, or addition of words compared to the reference text is penalized.
        \item Numbers, abbreviations, and proper nouns must be pronounced according to target-language conventions (e.g., Russian: ``1995'' $\rightarrow$ \textit{tysyacha devyatsot devyanosto pyat'}; French: ``M. Dupont'' $\rightarrow$ \textit{monsieur Dupont}; Kazakh: ``2024'' $\rightarrow$ \textit{eki myń jıyrma t\"{o}rt}).
        \item For Cyrillic scripts (Russian/Kazakh), the model must distinguish visually similar letters with distinct phonemes.
    \end{itemize}
    
    \item \textbf{Prosody \& Intonation:}
    \begin{itemize}
        \item Intonation contours must match sentence type (e.g., rising pitch for questions in French/Russian; declarative fall for statements).
        \item Sentence-level stress and phrasing should reflect syntactic boundaries. Arbitrary pauses within phrases or unnatural emphasis on function words are penalized.
    \end{itemize}
\end{itemize}

\subsection{Scoring and Preference Decision}
For each side-by-side trial, annotators assessed both Model A and Model B against the above criteria. The final preference was determined as follows:
\begin{itemize}
    \item If one model demonstrated significantly higher realism \textit{and/or} fewer linguistic errors, it was selected as preferred.
    \item If both models exhibited comparable levels of quality and error rates, the annotator selected ``Tie.''
    \item Each linguistic error (phoneme, stress, lexical, or prosodic) resulted in a proportional downgrade in the overall quality assessment for that sample.
    \item Annotators were instructed to prioritize \textit{intelligibility} over minor prosodic artifacts when making close calls.
\end{itemize}

\paragraph{Annotator Instructions.}
\begin{enumerate}
    \item Listen to both samples (A and B) for the same input text \textit{before} making a decision.
    \item Evaluate each dimension independently, then aggregate for the final preference.
    \item If audio quality issues (e.g., clipping, background noise) affect both samples equally, ignore them and focus on synthesis artifacts.
    \item When in doubt, select ``Tie'' rather than guessing.
\end{enumerate}

\paragraph{Statistical Reliability.}
\label{statisticalrel}
To quantify the reliability of the SBS preferences, we analyze the aggregate preference counts ($N=125$ per language; 5 annotators $\times$ 25 samples), excluding ties. We compute the SEAM preference proportion among non-tie judgments and report 95\% Wilson confidence intervals (Table~\ref{tab:sbs_significance}). An exact binomial test against the null hypothesis of equal preference ($p=0.5$) indicates that SEAM's preference is statistically significant across all three languages: Russian ($p=0.041$), Kazakh ($p=0.002$), and French ($p<0.001$). This confirms the robustness of the perceptual quality gains.
\begin{table}[h]
\centering
\small
\begin{tabular}{lccc}
\toprule
Lang.  & SEAM \% (Non-Tie) & 95\% CI & $p$-val \\
\midrule
Russian  & 61.4\% & [50.4, 71.4] & 0.041 \\
Kazakh   & 65.1\% & [55.6, 73.6] & 0.002 \\
French   & 70.6\% & [60.2, 79.2] & <.001 \\
\bottomrule
\end{tabular}
\caption{SBS preference statistics ($N=125$ per language; 5 annotators $\times$ 25 samples). ``SEAM \% (Non-Tie)'' represents the preference rate excluding ties. 95\% Wilson confidence intervals and exact binomial $p$-values (against $p=0.5$) demonstrate statistical significance across all evaluated languages.}
\label{tab:sbs_significance}
\end{table}

\section{Extended Related Work}
\label{sec:extended_related}

\paragraph{Cross-Lingual and Zero-Shot TTS Adaptation.}
Recent advances in zero-shot and cross-lingual TTS have dramatically expanded speaker and language coverage. Models such as VALL-E, VoiceCraft, and CosyVoice leverage discrete acoustic tokenizers and autoregressive or diffusion-based decoding to synthesize speech from short prompts \cite{valle2023,voicecraft2024,cosyvoice2024}. While highly effective for well-resourced languages, their adaptation to under-documented languages remains constrained by the need for paired text-speech data or language-specific phoneme inventories . Cross-lingual transfer methods typically rely on shared phoneme spaces, speaker-invariant conditioning, or meta-learning for rapid alignment \cite{yourtts2022}. However, these approaches still require substantial target-language fine-tuning to stabilize novel token embeddings, leaving the foundational step of vocabulary expansion underexplored.

\paragraph{Vocabulary Expansion and Subtoken Initialization.}
The standard pipeline for handling out-of-vocabulary (OOV) tokens decomposes unseen units into known subtokens via byte-pair encoding (BPE) or phoneme-grapheme mappings, then initializes embeddings using random draws or arithmetic mean pooling \cite{ruadapt2024}. While computationally trivial, mean pooling is permutation-invariant and discards the sequential articulatory structure essential to speech synthesis.

\textbf{Positioning.} SEAM bridges these research threads by introducing a unified, zero-audio pipeline that combines spectral low-pass filtering with closed-form ridge regression. Unlike iterative alignment methods, SEAM computes the optimal projection analytically, ensuring geometric stability from step zero. By explicitly retaining DFT phase information, it overcomes the permutation invariance of mean pooling while avoiding the data dependencies of contrastive alignment. This positions SEAM as a lightweight, theoretically grounded complement to modern zero-shot and cross-lingual TTS frameworks, specifically targeting the sample-efficiency bottleneck in vocabulary expansion.

\end{document}